\documentclass[letterpaper, 10 pt, conference]{ieeeconf}  

\IEEEoverridecommandlockouts                              

\usepackage{graphics} 
\usepackage{epsfig} 
\usepackage{times} 
\usepackage{amsmath} 
\usepackage{amssymb}  
\usepackage{fixmath}
\usepackage{multirow}
\usepackage[noadjust]{cite}
\usepackage{booktabs} 
\usepackage[colorlinks=true, urlcolor=blue]{hyperref} 

\usepackage{amsthm}
\newtheorem{theorem}{Theorem}
\newtheorem{definition}{Definition}
\newtheorem{proposition}{Proposition}
\newtheorem{corollary}{Corollary}

\newtheorem{remark}{Remark}

\let\labelindent\relax
\usepackage{enumitem}
\usepackage[linesnumbered, ruled, vlined]{algorithm2e}
\usepackage{xcolor}

\SetCommentSty{mycommfont}
\usepackage[font=footnotesize]{caption}
\usepackage[capitalize]{cleveref}
\crefformat{equation}{(#2#1#3)}
\Crefformat{equation}{Equation~(#2#1#3)}
\Crefname{equation}{Equation}{Eqs.}

\title{\LARGE \bf
STITCHER: \\ Real-Time Trajectory Planning with Motion Primitive Search
}

\author{Helene J. Levy and Brett T. Lopez
\thanks{Authors are with the VECTR Laboratory, University of California, Los Angeles, Los Angeles, CA, USA. {\tt\small \{hjlevy, btlopez\}@ucla.edu}}%
}

\begin{document}

\maketitle
\thispagestyle{empty}
\pagestyle{empty}
\setlength{\parskip}{0pt}
\addtolength{\topmargin}{0.05in}  

\begin{abstract}
Autonomous high-speed navigation through large, complex environments requires real-time generation of agile trajectories that are dynamically feasible, collision-free, and satisfy constraints.
Most modern trajectory planning techniques rely on numerical optimization because high-quality, expressive trajectories that satisfy constraints can be systematically computed.
However, strict requirements on computation time and the risk of numerical instability can limit the use of optimization-based planners in safety-critical situations.
This work presents an optimization-free planning framework called STITCHER that leverages graph search to generate long-range trajectories by stitching short trajectory segments together in real time. 
STITCHER is shown to outperform modern optimization-based planners through its innovative planning architecture and several algorithmic developments that make real-time planning possible.
Simulation results show safe trajectories through complex environments can be generated in milliseconds that cover tens of meters.
\end{abstract}

\vspace{0.1in}
\noindent \textbf{\small Code:} \href{https://github.com/vectr-ucla/stitcher}{\small https://github.com/vectr-ucla/stitcher}

\section{INTRODUCTION}
Planning collision-free and dynamically feasible trajectories through complex environments in real-time is a necessary capability for many autonomous systems.
As a result, trajectory planning has received considerable attention from the research community, but achieving the reliability and computational efficiency required for real-world, safety-critical applications remains a challenge. 
In particular, few methods have guarantees in terms of trajectory optimality and time/memory complexity without sacrificing trajectory expressiveness, length, or computation time. 
Our proposed approach addresses this gap by combining optimal control theory with graph search to generate near-optimal trajectories over long distances in real-time without online optimization.

Numerical optimization has emerged as the principal approach for trajectory design in autonomous systems. This is because it allows for natural expression of performance index and constraints, and high-quality solutions for complex problems.  
Continuous variable methods employ gradient descent to jointly optimize the coefficients of basis functions and waypoint arrival times \cite{Mellinger11:Minimum, Richter16:Polynomial,Oleynikova16:Continuous-time, Zhou2021:RAPTOR, Wang22:Geometrically, Ren22:Bubble}, while mixed-integer variable methods utilize integer variables to impose collision constraints along a discretized trajectory \cite{Deits15:Computing, Deits15:Efficient, Tordesillas22:FASTER, Marcucci23:Motion}.
Despite their popularity, optimization methods lack time complexity bounds that can be known \textit{a priori}, and can scale poorly with trajectory length, especially if integer variables are used.
Numerical stability can also be a problem with these methods.
\begin{figure}[t!]
  \centering
  \includegraphics[width=0.95\columnwidth]{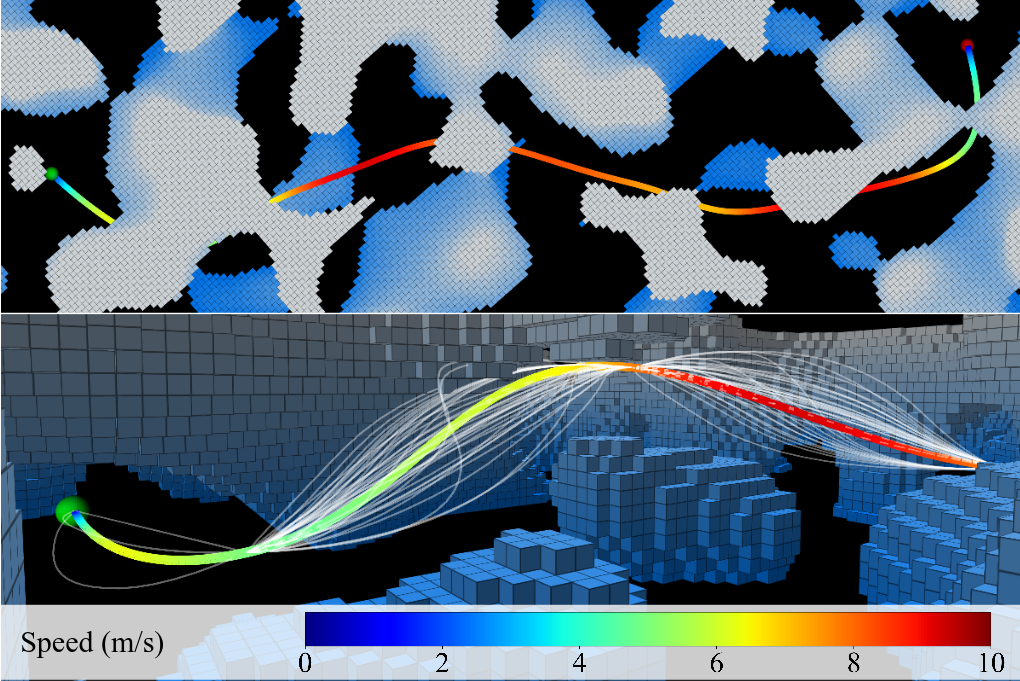}
  \caption{A trajectory (colored based on speed) generated by our proposed algorithm called STITCHER through a Perlin Noise environment. STITCHER searches over candidates motion primitives (white) to find a safe trajectory in real-time with time and memory complexity guarantees.}
  \label{fig:first_pic}
\end{figure}
A computationally efficient alternative is to continuously replan with a library of short-duration trajectories, i.e., motion primitives, that can be efficiently evaluated \cite{Mueller15:A_Computationally, Florence16:Integrated, Lopez17:Aggressive, Ryll19:Efficient}.
However, this framework can introduce myopic or suboptimal behavior that is exacerbated in large or complex environments. 
Subsequent work has attempted to pose the problem as a graph search with nodes and edges being desired states and motion primitives \cite{Liu17:Search_Based, Liu18:Search_Based, Zhou19:Robust, Foehn21:Alphapilot}.
Although this enables long-range trajectories, search times can be extremely high, and it is non-trivial to design an admissible heuristic that expedites the search.

In this work, we introduce a new trajectory planning algorithm called STITCHER, which enables real-time motion primitive search over long distances in complex environments.
STITCHER utilizes a novel three-stage planning architecture to generate expressive trajectories by \emph{stitching} motion primitives together.
Specifically, given a set of waypoints computed in the first stage, we create a velocity graph by sampling velocities at each waypoint, and employ dynamic programming to compute the cost-to-go for each node in the graph. 
The cost-to-go is then used as a heuristic to efficiently guide the motion primitive search in the third stage.
We also propose a greedy graph pre-processing step to form a compact motion primitive search graph.
We prove all graphs are finite, and that the proposed search heuristic is admissible.
These properties guarantee i) \emph{a priori} time and memory complexity bounds and ii) trajectory optimality with respect to the graph discretization set.
To further reduce computation time, we improve the collision checking procedure from \cite{Lopez17:Aggressive} by leveraging the known free space from previous trajectory evaluations, bypassing the rigidity and computational complexity of free space decomposition. 
Additionally, we show that employing a simple sampling procedure in the final search stage is effective at pruning trajectory candidates that violate complex state or actuator constraints.
STITCHER was extensively tested in two simulation environments, and compared with two state-of-the-art real-time optimization-based planners \cite{Tordesillas22:FASTER,Wang22:Geometrically}.
Results show that STITCHER consistently generates trajectories faster with comparable trajectory execution times.

\section{RELATED WORKS}
Optimization-based trajectory planners can be categorized using several criteria, but the clearest delineation is whether the method uses continuous or integer variables.
For methods that use only continuous variables, the work by \cite{Richter16:Polynomial} reformulated \cite{Mellinger11:Minimum} to jointly optimize over polynomial endpoint derivatives and arrival times for a trajectory passing through waypoints. 
Collisions were handled by re-optimizing the trajectory with added waypoints. 
Oleynikova et al.~\cite{Oleynikova16:Continuous-time} represented obstacles using an Euclidean Signed Distance Field (ESDF) which was incorporated into a nonconvex solver as a soft constraint.
Zhou et al.~\cite{Zhou2021:RAPTOR} used a similar penalty-based method but introduced a topological path search to escape local minima. 
An alternative approach is to decompose the occupied or free space into convex polyhedra \cite{Mellinger12:Mixed-integer, Deits15:Computing, Liu17:Planning} which can be easily incorporated as constraints in an optimization. 
The methods in \cite{Wang22:Geometrically, Ren22:Bubble} treat these constraints as soft while efficiently optimizing over polynomial trajectory segments that must pass near waypoints. 
One can also use the free-space polyhedra to formulate a mixed-integer program \cite{Deits15:Efficient, Tordesillas22:FASTER, Marcucci23:Motion} to bypass the nonconvexity introduced by having unknown waypoint arrival times, but at the expense of poor scalability with trajectory length and number of polyhedra.

Motion primitive planners have been proposed as an alternative to optimization-based planners to address computation and numerical instability concerns. 
Initial work on motion primitives for quadrotors leveraged differential flatness and known optimal control solutions to efficiently compute point-to-point trajectories \cite{Hehn11:Quadrocopter,Mueller15:A_Computationally}.
Later work employed motion primitives for receding horizon collision avoidance where  motion primitives were generated online by sampling desired end states, and selected based on safety and trajectory cost \cite{Florence16:Integrated,Lopez17:Aggressive, Lopez17:AggressiveFOV, Ryll19:Efficient, Dharmadhikari20:Motion}.
While computationally efficient, the behavior of these planners can be myopic, leading to suboptimal behavior in complex environments.
One way to address this behavior is to extend standard search-based algorithms, which typically use discrete action sets, to a lattice of motion primitives \cite{Dolgov10:Path,Pivtoraiko11:Kinodynamic,Liu17:Search_Based, Liu18:Search_Based, Zhou19:Robust, Andersson18:Receding}. 
The main issue with search-based motion primitive planners is the search space can become untractable, and with the non-triviality of constructing an admissible search heuristic, search times are not suited for real-time use. 
Recently, \cite{Foehn21:Alphapilot, Penicka22:Minimum, Romero22:Model} proposed an efficient waypoint-constrained minimum-time motion primitive search in velocity space using a double integrator model. 
The search is real-time but the resulting bang-bang acceleration profile is dynamically infeasible for aerial vehicles.
A final smoothing step, e.g., model predictive contouring control, is required to achieve sufficient trajectory smoothness \cite{Romero22:Model, Krinner24:MPCC++}.

\section{PROBLEM FORMULATION} 
In this work, we are concerned with solving the following trajectory planning problem 
\begin{align}
    \min_{\mathbold{u} \, \in \, \mathcal{U}} \quad &J = r(T) +  \int_{0}^{T} q(\mathbf{x},\mathbold{u}) \,  dt \label{eq:tpp}
 \\
    \text{s.t.} \quad &\dot{\mathbf{x}} = A\mathbf{x}+B\mathbold{u} \notag \\
    & \mathbf{x} \in \mathcal{X}_s, \ \mathbf{x} \notin \mathcal{X}_{obst},  \ \mathbold{u} \in \mathcal{U} \notag \\
    &\mathbf{x}(0) = \mathbf{x}_0, \  \mathbf{x}(T) = \mathbf{x}_f, \notag
\end{align}
where $\mathbf{x}\in \mathbb{R}^n$ is the state that must satisfy state $\mathcal{X}_s$ and obstacle (collision) $\mathcal{X}_{obst}$ constraints, $\mathbold{u} \in \mathbb{R}^m$ is the control input that must satisfy actuator constraints $\mathcal{U}$, $A \in \mathbb{R}^{n\times n}$ and $B\in \mathbb{R}^{n\times m}$ govern the system's dynamics, and $r : \mathbb{R}_+ \rightarrow \mathbb{R}_+$ and $q:\mathbb{R}^n \times \mathbb{R}^m \rightarrow \mathbb{R}_+$ are the terminal and stage cost, respectively. 
The goal is to find an optimal final time $T^*$ and feasible optimal state trajectory $\mathbf{x}^*(t)$ with a corresponding control input $\mathbold{u}^*(t)$ for $t \in [0,\, T^*]$ that steers the system from an initial state $\mathbf{x}_0$ to a desired final state $\mathbf{x}_f$ that minimizes the cost functional $J$.
While the dynamics are linear in \cref{eq:tpp}, differentially flat nonlinear systems, e.g., quadrotors, can be represented as a linear system with a state vector $\mathbf{x} = [\mathbold{r}, \,  {\mathbold{v}}, \, {\mathbold{a}}, \, \dots, \, \mathbold{r}^{(p\!-\!1)} ]^\top$ and control input $\mathbold{u} = \mathbold{r}^{(p)}$ where $\mathbold{r} = (x,\,y,\,z)^\top$ is the vehicle's position in some reference frame.

\subsection{Background: Motion Primitives}
We define motion primitives to be closed-form solutions to certain optimal control problems. 
In this work, we will restrict our attention to the following two optimal control problems. 
The formulations will be presented for a single axis, but can be repeated for all three position coordinates.

\begin{figure*}[t]
  \centering
  \includegraphics[width=\textwidth]{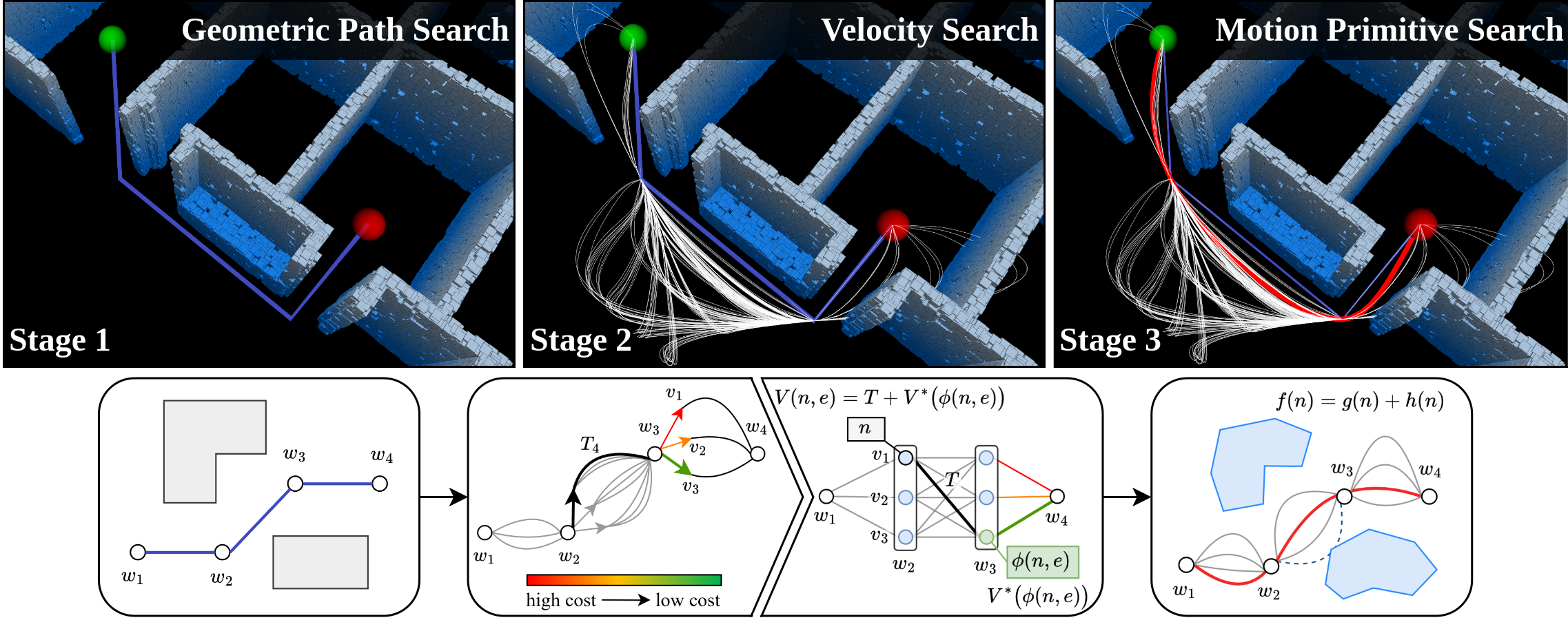}
  \caption{System architecture describing the three planning stages. In Stage 1, a sparse geometric path is found via A* search on a voxel occupancy grid. In Stage 2, a velocity state is introduced at each waypoint and dynamic programming is used to recursively solve for the cost-to-go at each node. In Stage 3, a full motion primitive search with collision/constraint checking informed by the cost-to-go computed in Stage 2 is performed to yield the final trajectory.}
  \label{fig:full_sys_arch}
\vskip -0.2in    
\end{figure*}

\textit{Minimum-Time Double Integrator:} Given an initial state $({s}_0, \, {{v}}_0) \in \mathbb{R}^2$ and desired final state $({s}_f, \, {{v}}_f) \in \mathbb{R}^2$, the minimum-time double integrator optimal control problem is 
\begin{align}
\min_{{u}} \quad & J = T \label{eq:double_min_time}
 \\
\text{s.t.} \quad &\ddot{{s}} = {u}, ~ |{u}|  \leq u_{max} \notag \\
& {s}(0) = {s}_0, \ {{v}}(0) = {{v}}_0 \notag \\
& {{s}}(T) = {{s}}_f,  \ {{v}}(T) = {{v}}_f, \notag
\end{align}
where the final time $T$ is free. 
The problem is known to have a bang-bang control profile detailed in \cite{Kirk04:Optimal}.
The control input switching times, which fully characterizes the solution, can be efficiently computed by solving a quadratic equation. 

\textit{Linear Quadratic Minimum-Time $p$-th Order Integrator:} 
Smooth trajectories can be generated by solving the linear quadratic minimum-time (LQMT) optimal control problem,
\begin{align}
    \min_{T,\,{u}} \quad & J = \rho \, T + \int_{0}^{T} {u}^2 \, dt \label{eq:triple_min_effort} \\
    \text{s.t.} \quad &{{s}}^{(p)} = {u} \notag \\
    & {{s}}(0) = {{s}}_0, \, {{v}}(0) = {{v}}_0, \dots, \, s^{(p\!-\!1)}(0) = s^{(p\!-\!1)}_0 \notag \\
    & {{s}}(T) = {{s}}_f, \, {{v}}(T) = {{v}}_f, \, s^{(k\!-\!1)}(T) \ \text{free for $3 \leq k \leq p$} \notag
\end{align}
where $\rho>1$ penalizes the final time. The final time $T$ and all terminal states except position and velocity are free. 
The final time can be found with a root-finding method 
such as QR algorithm \cite{Demmel97:Applied} 
because the cost functional can be expressed as a polynomial in terms of $T$ and the known boundary conditions, detailed in \cite{verriest2002linear,Liu18:Search_Based}.
State constraints are omitted from \cref{eq:triple_min_effort} as it is more efficient to prune many candidate trajectories once the final time is known, as discussed in \cref{sec:pruning}.

\section{METHODOLOGY}
STITCHER generates a full-state trajectory by {stitching} collision-free, dynamically feasible trajectory segments together through graph search.
At its core, STITCHER searches over closed-form solutions, i.e., motion primitives, to optimal control problems like those discussed above.
These solutions serve as a basis for the solution space to \cref{eq:tpp}.
To achieve real-time performance, STITCHER utilizes a three stage planning process where the final motion primitive search is guided by two other planners run sequentially (see \cref{fig:full_sys_arch}). 
In Stage 1 (left), A* algorithm is used to produce a sparse geometric path, i.e., waypoints, in the free space of the environment.
In Stage 2 (middle), nodes representing sampled velocities at the waypoints are formed into a velocity graph where dynamic programming is used to compute the minimum time path between nodes using a control-constrained double integrator model. 
This step is critical for constructing an admissible heuristic to guide the full motion primitive search, and is one of the key innovations that enables real-time performance. 
It is important to note that the optimal ``path" in velocity space is never used; computing the cost-to-go is the primary objective as it serves as an admissible heuristic for motion primitive search as shown in \cref{sec:DP}.
In Stage 3 (right), an A* search is performed over motion primitives using a higher-order dynamical model and the heuristic from Stage 2.
At this stage, position and all higher-order derivatives are considered, yielding a full state trajectory that can be tracked by the system.
Collisions and other state and control input constraints are also checked in this stage.
The remainder of this section expands upon each component of STITCHER.

\subsection{Stage 1: Forward Geometric Path Search}
STITCHER requires a sequence of waypoints that guides the motion primitive search by limiting the size of the search space. 
This can be done by generating a collision-free geometric path (see \cref{fig:full_sys_arch} left) through the environment with A* search or any other discrete graph search algorithm where the environment is represented as a 3D voxel occupancy grid.
Let the collision-free, geometric path generated by a discrete graph search algorithm be composed of points $\mathcal{O}= \{\mathbold{o}_1,\mathbold{o}_2, ..., \mathbold{o}_H\}$ where $\mathbold{o}_i \in \mathbb{R}^3$. 
The set of points $\mathcal{O}$ is further pruned to create a sparse set of waypoints $\mathcal{W} = \{\mathbold{w}_1, \mathbold{w}_2, ..., \mathbold{w}_N\}$ where $N \leq H$ and $\mathbold{w}_i \in \mathbb{R}^3$. 
Sparsification is done by finding the minimal set of points in $\mathcal{O}$ that can be connected with collision-free line segments.

\subsection{Stage 2: Backward Velocity Search} \label{sec:phase_2_vel_graph}
The ordered waypoint set $\mathcal{W}$ found in Stage 1 only provides a collision-free geometric path through the environment.
In other words, the velocity, acceleration, and higher-order states necessary for tracking control are not specified.
We propose creating a velocity graph (see \cref{fig:full_sys_arch} middle) where each node in the graph is defined by a position and velocity.
The positions are restricted to waypoint locations and $M$ velocities are sampled at each waypoint. 
More explicitly, for each waypoint $\mathbold{w}_i \in \mathcal{W}$, we sample a set of velocities $\mathcal{V} = \{\mathbold{v}_1, ..., \mathbold{v}_M\}$, where $\mathcal{V}$ is composed of candidate velocity magnitudes $\mathcal{V}_m$ and directions $\mathcal{V}_d$. 
With the ordered waypoint $\mathcal{W}$ and sampled velocity $\mathcal{V}$ sets, we create a velocity graph $\mathcal{G} = (\mathcal{N}, \mathcal{E})$, where node $n \in \mathcal{N}$ is a given position and sampled velocity, i.e., $n = (\mathbold{w}_i,\,\mathbold{v}_j)$ with $\mathbold{w}_i \in \mathcal{W}$ and $\mathbold{v}_j \in \mathcal{V}$, and edge $e \in \mathcal{E}$ is the \emph{double integrator control-constrained minimum-time} motion primitive $\mathbold{r}(t)$ from \cref{eq:double_min_time} that connects neighboring nodes. 
At this stage, collision and state constraints are not enforced to prevent candidate trajectories from being eliminated prematurely.

We recursively compute and store the ranked list of cost-to-go's $V_d:\mathcal{N}\times \mathcal{E} \rightarrow \mathbb{R}_+$ for each node $n \in \mathcal{N}$ and all connecting edges $e \in \mathcal{E}_n$ of $n$ where
\begin{align}
\label{eq:bellman}
    V_d(n,e) = \ell(n,e) + V_d^*\big(\phi(n,e)\big) ~~ \forall  e\in \mathcal{E}_n,
\end{align}
with the optimal cost-to-go $V_d^*(n) = \min_{e\in \mathcal{E}_n} V_d(n,e)$, the cost of taking edge $e$ from node $n$
being $\ell(n,e)$, and the node reached by taking edge $e$ being $\phi(n,e)$. 
The cost of taking an edge is given by $\ell(n,e) = T^*_d(n,e)$, where $T^*_d(n,e)$ is the minimum time of trajectory $\mathbold{r}(t)$ connecting the states of node $n$ to the states of $\phi(n,e)$.
Minimizing \cref{eq:bellman} is the well-known Bellman equation, which is guaranteed to return the optimal cost-to-go.
In \cref{sec:DP} we prove that $V_d^*(n)$ for each node in graph $\mathcal{G}$ is an admissible heuristic for an A* search over a broad class of motion primitives.

\subsection{Stage 3: Forward Motion Primitive Search}
\label{sec:phase_3_motion_search}
The cost-to-go's computed in Stage 2 for the sampled velocities at each waypoint serve as an admissible heuristic (see \cref{def:admissable}) that guides an efficient A* search over motion primitives.
The motion primitives can be generated using any chain of integrators model of order at least two so long as i) the initial and final position and velocities match those used to construct the velocity graph $\mathcal{G}$ and ii) the allowable acceleration is maximally bounded by $u_{max}$ given in \cref{eq:double_min_time}.
The motion primitive search graph is denoted as $\mathcal{G}_{mp} = (\mathcal{N}_{mp},\, \mathcal{E}_{mp})$ where $\mathcal{N}_{mp}$ is the set of nodes, each corresponding to a state vector, and $\mathcal{E}_{mp}$ is the set of edges, each corresponding to a motion primitive that connects neighboring nodes.
A* search is used to meet real-time constraints where the search minimizes the cost $f(n) = g(n) + h(n)$ where $n\in \mathcal{N}_{mp}$ is the current node, $g: \mathcal{N}_{mp} \rightarrow \mathbb{R}_+$ is the cost from the start node $n_s$ to node $n$, and $h: \mathcal{N}_{mp} \rightarrow \mathbb{R}_+$ is the estimated cost from the current node $n$ to the goal node $n_g$.  
In the context of optimal control, $g$ is the cost accrued, i.e., the running cost, for a path from $n_s$ to $n$ whereas $h$ is the estimated cost-to-go, i.e., the estimated value function $V^*$, from $n$ to $n_g$.
In this stage, collision and state constraints are checked for each candidate motion primitive to ensure safety; the methodology for both is discussed in \cref{sec:pruning}.
\begin{figure}[t]
  \centering
  \includegraphics[width=\columnwidth]{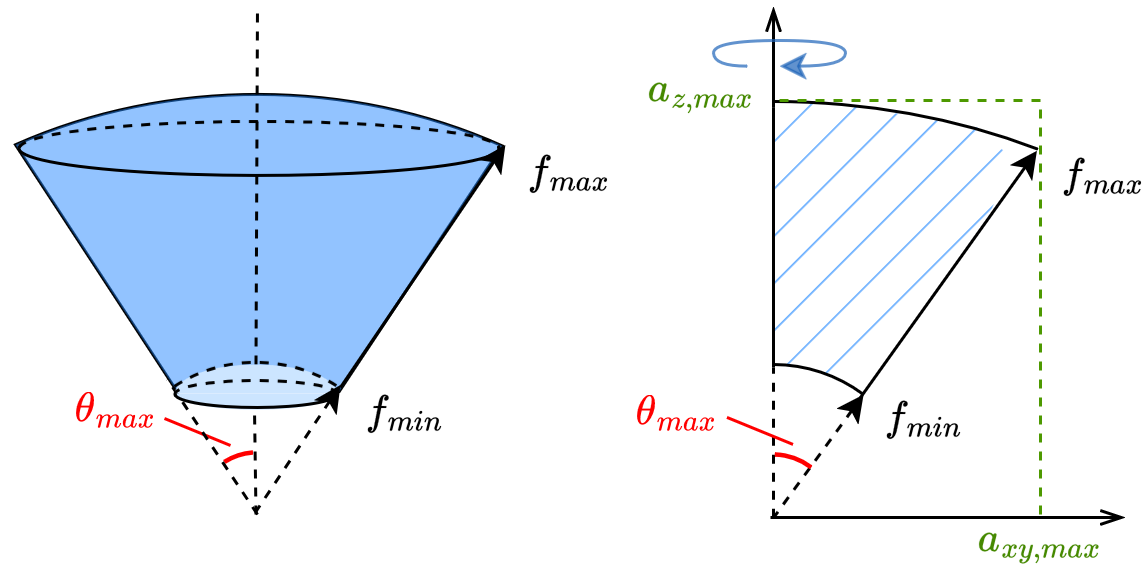}
  \caption{The achievable mass-normalized thrust (nonconvex) of a aerial VTOL vehicle with limits on minimum thrust $f_{min}$, maximum thrust $f_{max}$, and maximum tilt $\theta_{max}$. }
  \label{fig:coupled_constr}
\vskip -0.1in
\end{figure}

\subsection{Pruning Infeasible \& In-Collision Motion Primitives}
\label{sec:pruning}

STITCHER guarantees safety by pruning motion primitives from the final search that violate constraints or are in collision.
For state and actuator constraints, many optimization-based planning approaches approximate the true physical constraints of the system with simple convex constraints, e.g., $\Vert \mathbold{{v}}\Vert_\infty \leq v_{max},~\Vert \mathbold{{a}}\Vert_\infty \leq a_{max},$ etc., to reduce computational complexity.
When polynomials are used to represent the optimal trajectory, imposing a convex hull constraint on the polynomial is one method to enforce such constraints \cite{Zhou19:Robust, Tordesillas22:FASTER}.
However, many of these approximations are made only to simplify the resulting optimization problem and might not accurately reflect the actual physical constraint, which can lead to conservatism.
STITCHER has the freedom to use different methods to enforce state and actuator constraints, but we uniformly sample candidate trajectories in time to check for constraint violations as it was found to be effective and efficient.
Sampling avoids mistakenly eliminating safe trajectories, and the observed computation time was comparable to or better than using convex hulls.
Critically, sampling allows for the inclusion of more complex constraints, such as those that couple multiple axes. 
Examples are 
\begin{equation}
    \label{eq:state_constr}
    \begin{aligned}
        \text{Thrust Magnitude:} & ~~ 
        0 \leq f_{min} \leq \|\boldsymbol{f}\|_2 \leq f_{max} \\
        \text{Thrust Tilt Angle:} & ~~ \|\boldsymbol{f}\|_2 \cos(\theta_{max}) \leq f_z \\
        \text{Linear Velocity:} & ~~ \|\boldsymbol{v}\|_2 \leq v_{max} \\ 
        \text{Angular Velocity:} & ~~ \|\boldsymbol{\omega}\|_2 \leq \omega_{max},
    \end{aligned}
\end{equation}
where we note differential flatness can be leveraged to express the angular velocity constraint in terms of derivatives of position.
\Cref{fig:coupled_constr} depicts the achievable mass-normalized thrust of a VTOL vehicle given thrust and tilt constraints in \eqref{eq:state_constr}. 
The constraints are nonconvex making it difficult to include in real-time optimization-based planners without some form of relaxation, e.g., as in \cite{acikmese2007convex} for a double integrator, which is tight, or a more conservative relaxation. 

\begin{figure}[t]
  \centering
  \includegraphics[width=\columnwidth]{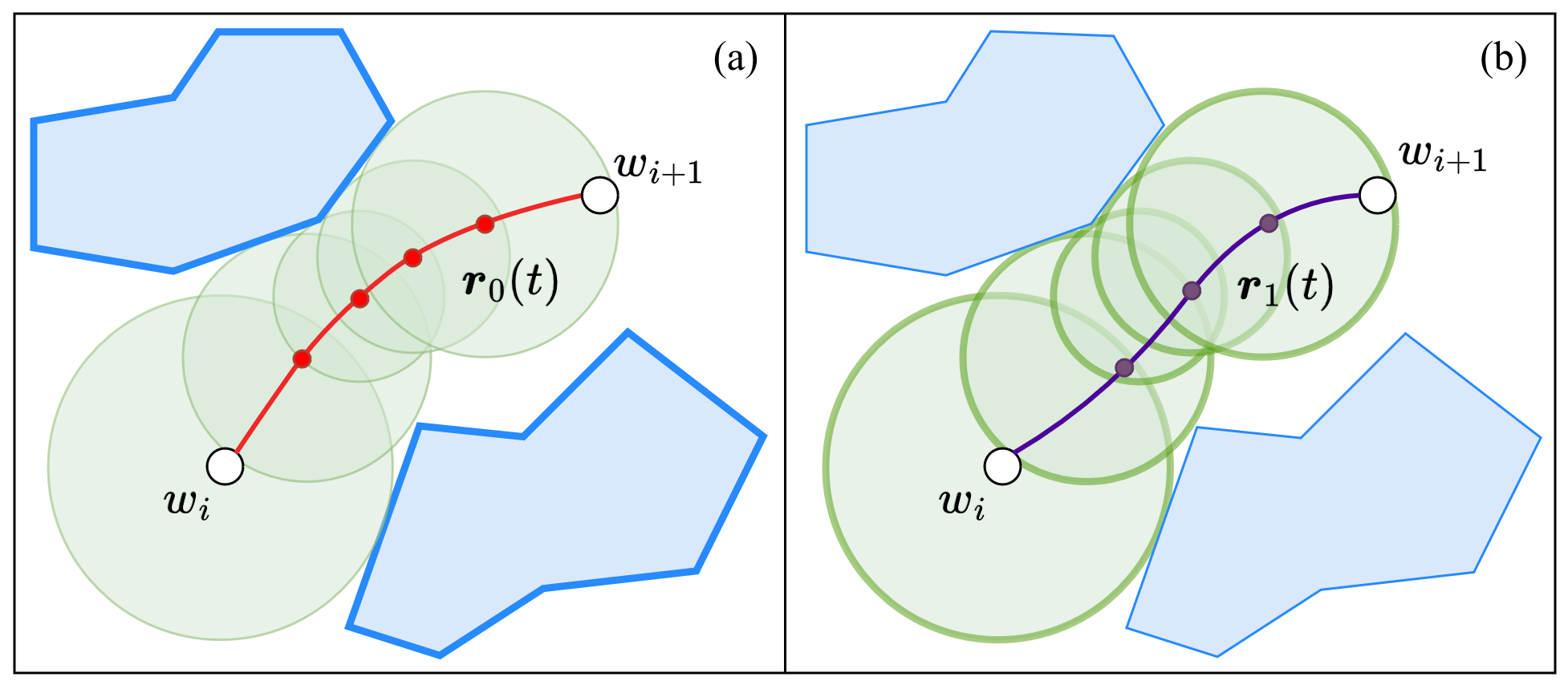}
  \caption{Removing redundant collision checks. (a): Motion primitive $\mathbold{r}_0(t)$ checks for collisions using \cite{Lopez17:Aggressive}. (b): Sampled points of $\mathbold{r}_1(t)$ are checked to lie within obstacle-free regions derived from $\mathbold{r}_0(t)$ calculations.}
  \label{fig:smart_collision}
\vskip -0.1in   
\end{figure}

An efficient collision checking strategy was devised by constructing a safe set of spheres resulting from a sampling-based collision checking approach proposed in \cite{Lopez17:Aggressive}.
The core idea from \cite{Lopez17:Aggressive} is that a trajectory can be intelligently sampled for collisions by estimating the next possible ``time-of-collision" along the trajectory by combining obstacle proximity and the vehicle's maximum speed.
Leveraging this idea, further computation time savings can be achieved by storing and reusing nearest neighbor queries.
\Cref{fig:smart_collision}a depicts that for the first candidate motion primitive connecting two successive waypoints, we use the strategy from \cite{Lopez17:Aggressive} while also storing the resulting set of safe, obstacle-free spheres $\mathcal{S}$.
For subsequent motion primitives between the same waypoint pair (see \cref{fig:smart_collision}b), a nearest neighbor query is only done if the primitive is expected to leave the set $\mathcal{S}$.
For a point found to be within a certain sphere, the next possible ``time-of-collision" is when the trajectory intersects the edge of the sphere, which can be estimated by assuming the trajectory emanates radially from the center of the sphere at maximum velocity.
The process is repeated until the final time horizon $T$ is reached.
Unlike spherical safety corridors, our safe set is only used to avoid repeated calculation, and allows for on-the-fly addition of collision-free spheres.
STITCHER thus has the flexibility to create and check candidate trajectories without being restricted to pre-defined safety spheres. 

\begin{figure}[t]
 \centering
 \includegraphics[width=\columnwidth]{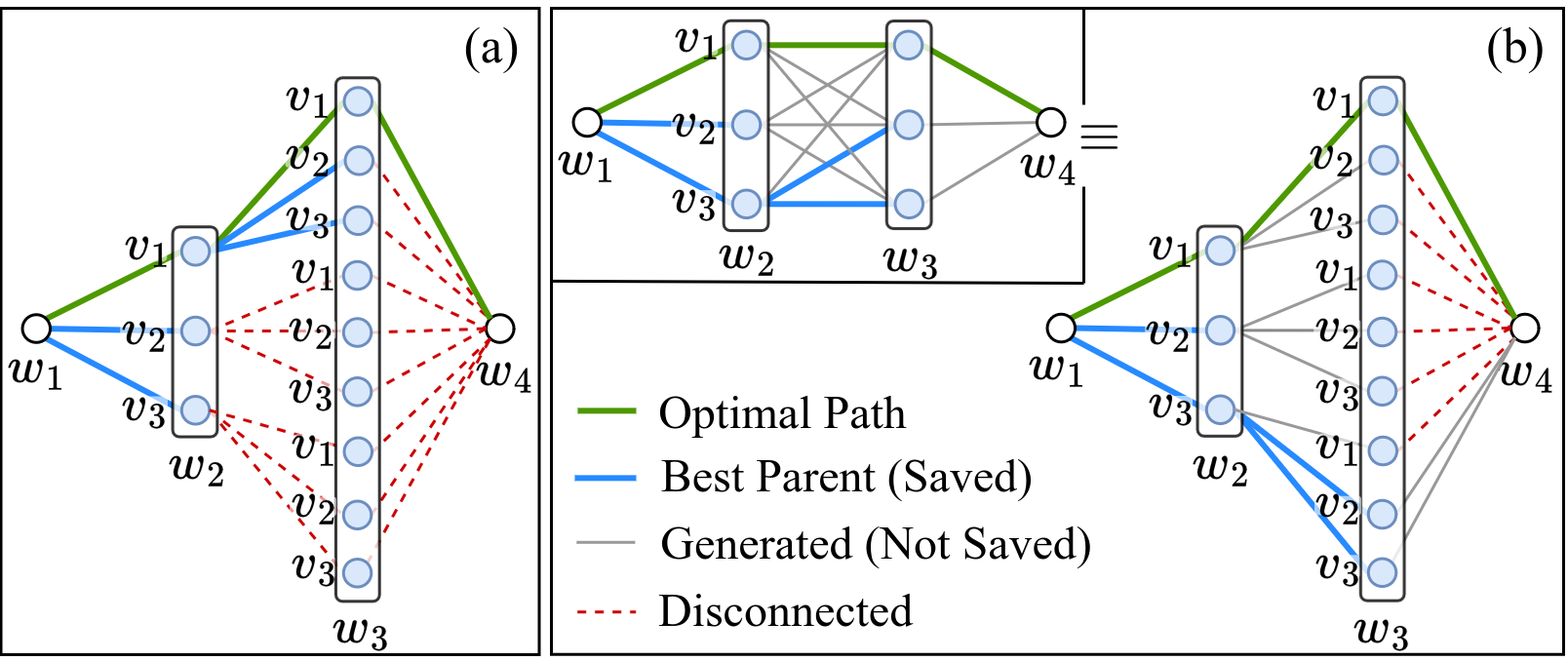}
\caption{Comparison between the graph of (a) a traditional greedy algorithm and (b) the greedy pre-processing step used by STITCHER.}
\label{fig:greedy_graph} 
\vskip -0.1in
\end{figure}

\subsection{Motion Primitive Search Graph with Triple Integrator}
In many applications, a triple integrator model for generating motion primitives is sufficient because 
discontinuities in jerk typically do not severely degrade tracking for most aerial vehicles. 
Motion primitives in our formulation are derived imposing a free terminal acceleration in \cref{eq:triple_min_effort}.
If the acceleration at each node, i.e., the final acceleration, ${\mathbold{a}}_f$, is free and graph nodes are represented by a waypoint-velocity-acceleration tuple, the number of edges grows exponentially with respect to the number of waypoints.
Our formulation employs a greedy pre-processing step in which the motion primitive search graph $\mathcal{G}_{mp}$ is identical in size to the velocity graph $\mathcal{G}$ (graph size detailed in Section \ref{sec:vel_graph_size}). 
This formulation offers an advantage in terms of computational efficiency, as a full-state trajectory is generated while the graph size is restricted by only the number of sampled velocities. 
Excluding acceleration information when creating the graph assumes that the optimal stitched trajectory is only weakly dependent on acceleration at each waypoint.
\Cref{fig:greedy_graph} shows the greedy graph pre-processing step (right) maintains more edges than a traditional greedy algorithm (left). 

\section{THEORETICAL ANALYSIS}
In this section we 
prove STITCHER has bounded time and memory complexity by showing the velocity and motion primitive graphs are finite. 
We also 
show STITCHER is complete and optimal by proving the heuristic used in the motion primitive search is admissible.

\subsection{Velocity Graph Complexity} \label{sec:vel_graph_size}
The following proposition proves the size of the velocity graph $\mathcal{G}$ is finite and solely depends on the number of waypoints and sampled velocities; a property that also holds for the motion primitive graph $\mathcal{G}_{mp}$ by extension.
This result is critical as a finite graph yields \emph{known time complexity} for the motion primitive search.
In other words, an upper bound can be placed on the computation time of the planner given known quantities. 
This is in contrast to optimization-based methods where the time complexity depends on the number of iterations required to converge---which cannot be known \textit{a priori}---so the time to compute a trajectory via optimization does not have an \textit{a priori} bound.

\begin{proposition} \label{prop_bwd_size}
For $N$ waypoints and $M$ sampled velocities, the number of nodes $|\mathcal{N}|$ and edges $|\mathcal{E}|$ in graph $\mathcal{G}$ is
\begin{align}
|\mathcal{N}| &= (N-2)M + 2, \label{eq:bwd_total_nodes}\\
|\mathcal{E}| &= (N-3)M^2 + 2M ~~ \text{for}~N > 2.
\label{eq:bwd_total_edges}
\end{align}
\end{proposition}
\begin{proof}
Using \cref{fig:full_sys_arch} (middle), the start and goal nodes contribute 2 nodes to the graph $\mathcal{G}$. 
For intermediate waypoints, given $M$ sampled velocities, there are $M$ nodes per waypoint.
As a result, $|\mathcal{N}| = (N-2)M + 2$ which is \cref{eq:bwd_total_nodes}.
For each edge, we consider the transition to successive waypoints. 
Ignoring the trivial $N = 2$ case where $|\mathcal{E}| = 1$, there are $M$ connections between the start node and next waypoint, which also has $M$ nodes. 
The same applies for connecting waypoint $w_{N-1}$ to the goal node, resulting in a total of $2M$ edges. 
For all other intermediate waypoint pairs, $M$ nodes connect to $M$ nodes at the next waypoint so there are $M^2$ edges. 
The total number of edges is then \cref{eq:bwd_total_edges}.
\end{proof}

\begin{corollary}
\label{cor:lqmt}
    The size of the motion primitive graph $\mathcal{G}_{mp}$ using Linear Quadratic Minimum Time (LQMT) motion primitives with free terminal acceleration for a triple integrator is identical to the velocity graph $\mathcal{G}$.
\end{corollary}

\begin{proof}
    The proof is immediate since the terminal acceleration is free so $N$ and $M$ are identical for both graphs.
\end{proof}

\subsection{Admissible Heuristic for Motion Primitive Search}
\label{sec:DP}
Heuristics are critical for speeding up graph search by incentivizing the search to prioritize exploring promising nodes.
For example, in A* search, the next node explored is selected based on minimizing the cost $f(n) = g(n) + h(n)$, where $g$ is the stage cost to get from the start node $n_s$ to node $n$, and $h$ is a heuristic estimate of the remaining cost to reach the goal node $n_g$.
A* search is guaranteed to find an optimal solution so long as the heuristic function $h$ is admissible (see \cref{def:admissable}) \cite{Russell16:Artificial}.
Below, we prove the cost-to-go $V^*$ for each node in the velocity graph $\mathcal{G}$ calculated in Stage 2
is an admissible heuristic for an A* search over motion primitives of any higher-order chain of integrators.
 
\begin{definition}[{\cite{Russell16:Artificial}}] 
\label{def:admissable}
A function $h : \mathcal{N} \rightarrow \mathbb{R}$ is an admissible heuristic if for all $n \in \mathcal{N}$ then $h(n) \leq h^*(n)$, 
where $h^*$ is the optimal cost from $n$ to the goal node $n_g$.
\end{definition}

\begin{proposition}
\label{prop:c2g_opt}
    Consider the optimal control problem
    \begin{align}
    \label{eq:gen_opt}
            \min_{T,\,\mathbold{u}} \quad & J =  \rho \, T + \int^T_0 q(\mathbold{r},\mathbold{{v}},\dots,\mathbold{u}) \, dt  \\ 
            \mathrm{s.t.} \quad &  \mathbold{r}^{(p)} = \mathbold{u}, \ c(\mathbold{a}) \leq 0 \notag \\
            & \mathbold{r}(0) = \mathbold{r}_0, \, \mathbold{v}(0) = \mathbold{v}_0, \, \dots , \mathbold{r}^{(p\!-\!1)}(0) = \mathbold{r}_0^{(p\!-\!1)} \notag\\
            & \mathbold{r}(T) = \mathbold{r}_f, \, \mathbold{v}(T) = \mathbold{v}_f, \,  \mathbold{r}^{(k\!-\!1)}(T) \ \mathrm{free} \ \mathrm{for} \ 3 \leq k \leq p \notag
    \end{align}
    where $q$ is a positive definite function, the system is at least second order ($p\geq 2$), and the position and velocity boundary conditions are identical to those of \cref{eq:double_min_time}, with all other boundary constraints free to specify. 
    If $u_{max}$ in \cref{eq:double_min_time} is the maximum possible acceleration achievable in a given axis imposed by $c(\mathbold{a}) \leq 0$, then the optimal cost-to-go $V^*$ from the initial conditions for \cref{eq:gen_opt} satisfies $V^*  \geq  \rho \, T_d^*$ where $T_d^*$ is the optimal final time for \cref{eq:double_min_time}.
\end{proposition}
\begin{proof}
    First, consider the case when $p=2$. 
    For a given axis, if $u_{max}$ is chosen so that it exceeds the allowable acceleration imposed by $c(\boldsymbol{a}) \leq 0$, e.g., $u_{x,max} \geq \max_{a_x}c({\boldsymbol{a}})$ (see \cref{fig:coupled_constr}), then the optimal final time $T^*$ for $\cref{eq:gen_opt}$ will always be greater than that of \cref{eq:double_min_time} even when $q = 0$.
    Specifically, when $q = 0$, one can show the optimal final time for \cref{eq:double_min_time} increases as $u_{max}$ decreases. 
    Moreover, $T^*_d$ for \cref{eq:double_min_time} is guaranteed to exist and be unique \cite{Kirk04:Optimal}.
    Hence, by appropriate selection of $u_{max}$, we can ensure $T^* \geq T^*_d$ always, where equality holds when $q = 0$ and $c(\boldsymbol{a})$ is a box constraint. 
    If $q \neq 0$, then it immediately follows that $T^* > T^*_d$ because $q$ is positive definite by construction.
    Now consider the case when $p > 2$.
    We can deduce $V^*  >  \rho \, T^*_d$ by contradiction.
    Specifically, assume $T^* = T_d^*$ for $p > 2$.
    This would require $\boldsymbol{a}$ to be discontinuous in order to match the bang-bang acceleration profile of \cref{eq:double_min_time}.
    However, \cref{eq:gen_opt} is a continuous-time linear system that will not exhibit discrete behaviors, e.g., jumps, so it is mathematically impossible to generate an optimal control sequence where the acceleration profiles for \cref{eq:gen_opt,eq:double_min_time} will be identical. 
    It can then be concluded $V^* > \rho \, T^*_d$ for $p > 2$.
    Therefore, $V^* \geq \rho \, T^*_d$ for $p \geq 2$, as desired.
\end{proof}

\begin{remark}
    \cref{prop:c2g_opt} also holds when inequality state or actuator constraints in \cref{eq:gen_opt} are present, and when the terminal desired states are specified rather than free.
\end{remark}

The main result of this section can now be stated. 

\begin{theorem}
\label{theorem:admissible}
    The optimal cost-to-go for the minimum-time input-constrained double integrator optimal control problem \cref{eq:double_min_time} is an admissible heuristic for motion primitive search where the primitives are solutions to the optimal control problem of the form \cref{eq:gen_opt}.
\end{theorem}

\begin{proof}
    Let $\mathcal{G} = (\mathcal{N},\mathcal{E})$ be a graph with nodes being sampled velocities at waypoints and edges being the time-optimal trajectories using an input-constrained double integrator.
    Further, let $\mathcal{G}_{mp} = (\mathcal{N}_{mp},\mathcal{E}_{mp})$ be a graph with nodes being sampled velocities, accelerations, etc. at waypoints and edges being trajectories that are solutions to \cref{eq:gen_opt}. 
    Using the Bellman equation, the optimal cost-to-go $V_{mp}^*(n)$ for any $n \in \mathcal{N}_{mp}$ can be computed recursively.
    Using \cref{prop:c2g_opt}, $V_{mp}^*(n) \geq V_d^*(n^\prime)$ by induction where $V_d^*$ is the optimal cost-to-go for the minimum-time input-constrained double integrator with $n^\prime \in \mathcal{N}$. 
    Recognizing $\mathcal{N} \subseteq \mathcal{N}_{mp}$, $V_d^*(n^\prime)$ can be rewritten as $V_d^*(n)$.    
    Setting $h^*(n) = V^*_{mp}(n)$ and $h(n) = V_d^*(n)$, it can be concluded $h(n) \leq h^*(n)$.
    Therefore, by \cref{def:admissable}, the optimal cost-to-go computed for $\mathcal{G}$ is an admissible heuristic for the motion primitive search over $\mathcal{G}_{mp}$.
\end{proof}

The importance of \cref{theorem:admissible} follows from the fact that searching a graph with an admissible heuristic is \emph{guaranteed} to return the optimal path through the graph \cite{Russell16:Artificial}, and can significantly improve search efficiency. 
The effectiveness of the proposed heuristic is analyzed in \cref{sec:results}.

\begin{figure}[t]
 \centering
 \includegraphics[width=\columnwidth]{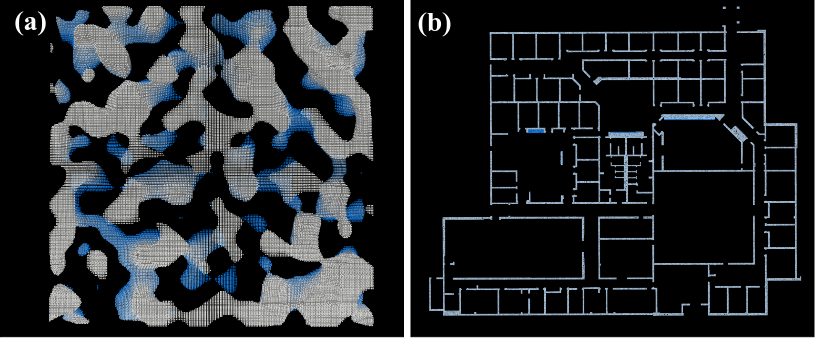}
\caption{Simulation test environments. (a): Perlin Noise environment (b): Willow Garage environment.}
\label{fig:env_type} 
\vskip -0.1in
\end{figure}

\section{SIMULATION RESULTS}
\label{sec:results}

Simulation experiments were completed in a Perlin Noise and the Willow Garage environment, both with a volume of approximately $50\times50\times5$ m (see \cref{fig:env_type}).
Geometric paths with $N = 4,\,6,\,8$ waypoints were found for different start and end locations in each environment.
For all experiments, we imposed $f_{min} = 0.85 \text{ m/s}^2$, $f_{max} = 18.75 \text{ m/s}^2$, $\theta_{max} = 60^\circ$, $\omega_{max} = 6 \text{ rad/s}$, $v_{max} = 10 \text{ m/s}$, and a time penalty $\rho = 1000$.
STITCHER requires a discrete velocity set $\mathcal{V}$ which is composed of a set of magnitudes and directions.
At each waypoint $\mathbold{w}_{i}$, the velocity direction set $\mathcal{V}_d$ is defined by the center and boundaries of a $20^\circ$ cone. We define the center of the cone as the vector that points from the previous waypoint $\mathbold{w}_{i-1}$ to the next waypoint $\mathbold{w}_{i+1}$.
For magnitudes, we use the set $\mathcal{V}_m = \left\{0, ~0.25\,v_{max}, ~0.5\,v_{max}, ~0.75\,v_{max}, ~ v_{max}\right\}$ for our analysis unless otherwise indicated.
All reported times are from tests run on an 11$^\mathrm{th}$ generation Intel i7 laptop.

\begin{table}[t]
\caption{Generated Edges for Heuristic Evaluation}
\vskip -0.1in  
\label{tab:heuristic_edge_explored_time}
\begin{center}
\begin{tabular}{c|c|c|c|c|c}
\hline
\multirow{2}{*}{Map} & \multirow{2}{*}{$N$} & \multirow{2}{*}{Total Edges} & \multicolumn{2}{|c|}{Edges Generated} & \multirow{2}{*}{\% Red.} \\
& & & Dijkstra & STITCHER & \\
\hline
\multirow{3}{*}{\parbox{0.9cm}{\centering Perlin Noise}} & 4 & 1023 & 957 &  {742} & 22 \\
& 6 & 2945 & 2720 & {2139} & 21\\
& 8 & 4867 & 4804 & {4082} & 15 \\
\hline
\multirow{3}{*}{\parbox{0.9cm}{\centering Willow Garage}} & 4 & 1023 & 896 & {893} & 0.3 \\
& 6 & 2945 & 2108 & {1843}  & 13\\
& 8 & 4867 & 3955 & {2946}  & 26\\
\hline
\end{tabular}
\end{center}
\vskip -0.1in  
\end{table}

\begin{figure}[t!]
 \centering
 \includegraphics[width=\columnwidth]{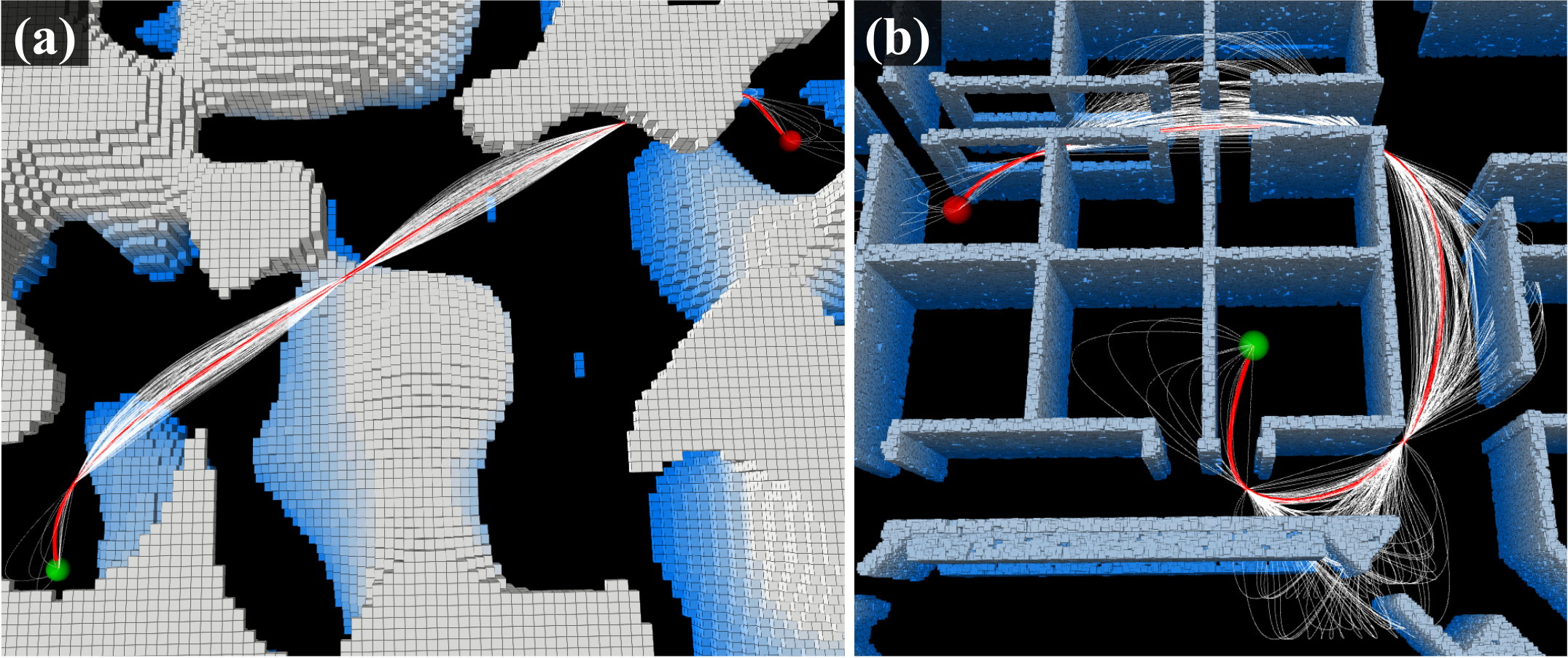}
 \caption{Final trajectories (red) through (a): Perlin Noise and (b) Willow Garage environment. Trajectory options (white) which inform the heuristic are more likely to be in collision in the Willow Garage due to tight corridors.}
\label{fig:environment_comparison}
\vskip -0.1in    
\end{figure}

\subsection{Heuristic Benchmarking}
\label{sec:heuristic_benchmark}
The quality of the heuristic used to guide STITCHER can be quantified by comparing the number of edges (number of motion primitives), generated by STITCHER to an offline version of STITCHER that runs Dijkstra's algorithm rather than A*.
The number of edges created is a better evaluation metric than nodes explored because motion primitive generation and evaluation is the main source of computation time.
\Cref{tab:heuristic_edge_explored_time} shows the number of edges created for STITCHER and Dijkstra's algorithm using execution time as the edge cost. The velocity magnitude and direction sets were kept constant across both planners with $|\mathcal{V}_m|=11$ containing speeds in the interval $[0,\,v_{max}]$ and $|\mathcal{V}_d|= 3$. STITCHER generates an average of 20\% fewer edges in the Perlin Noise environment and 13\% fewer in the Willow Garage environment.
The reduced effectiveness of the heuristic in the latter test case is attributed to narrower corridors, resulting in a greater number of motion primitives being in collision (see \cref{fig:environment_comparison}). 
The reduction in explored nodes shows that the heuristic is effective, but its performance depends on the environment. 
Note that Dijkstra's algorithm does not generate the maximum possible number of edges because nodes become disconnected if motion primitives are found to be in collision or exceed state constraints.

\begin{figure}[t]
 \centering
 \includegraphics[width=\columnwidth,trim={0 0.25cm 0 0},clip]{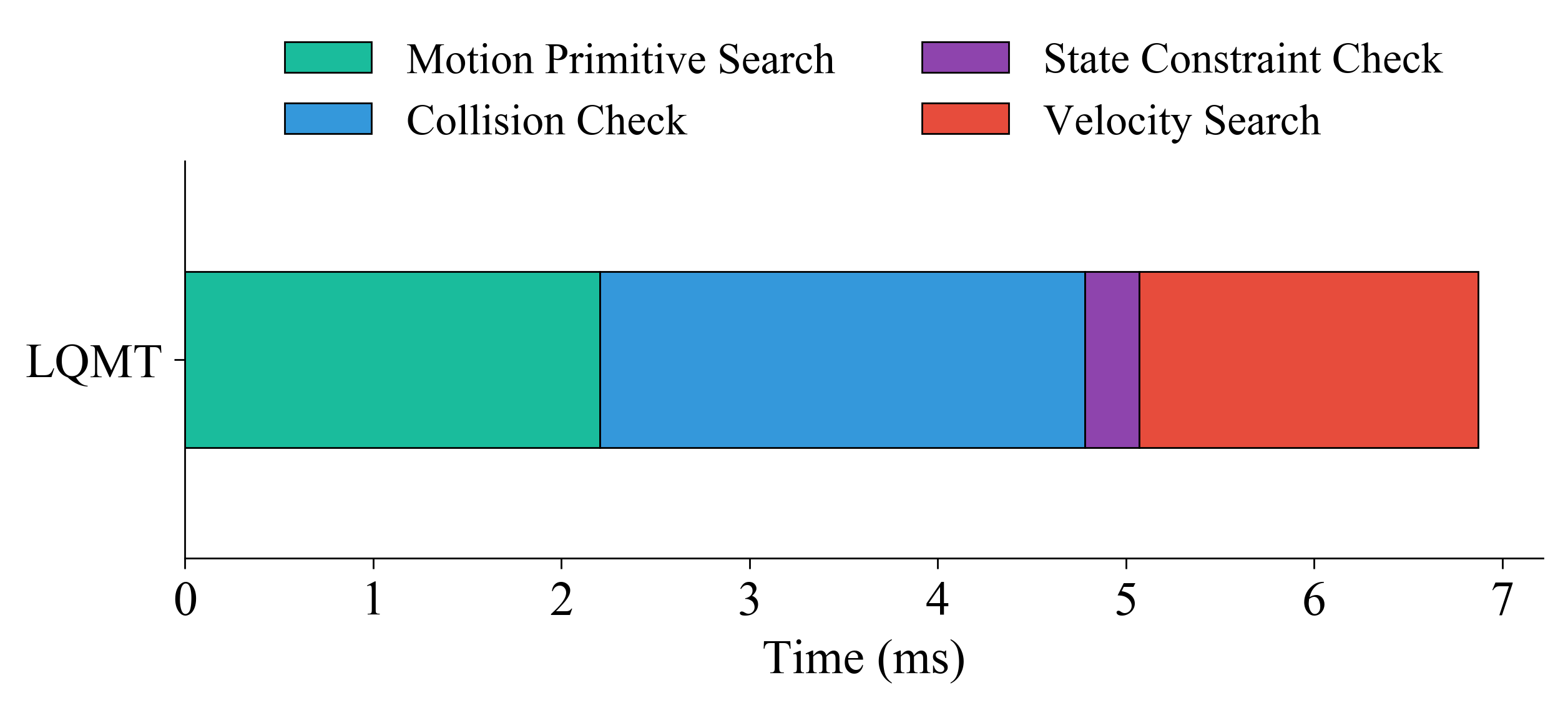}
 \caption{The average contribution of different path planning components.}
\label{fig:planning_time}
\vskip -0.1in    
\end{figure}

\subsection{Timing Analysis}
\Cref{fig:planning_time} shows the average computation time of the different processes of STITCHER across the six tested trials.
The average time to perform the velocity graph search is 1.8 ms, compared to 2.2 ms for the motion primitive search.
Although both searches have the same graph size, it is important to note that the motion primitives from \eqref{eq:triple_min_effort} are more time-consuming to generate than the minimum time trajectories used in the velocity graph. 
Hence, the similar computation times is from the search heuristic reducing the number of edges generated.
The low computation time of the velocity search further indicates its effectiveness in computing an informative admissible heuristic for the motion primitive search.
Finally, constraint checking with uniform samples every 0.1 s took an average of only 0.3 ms, demonstrating the method's high efficiency.

\subsection{Comparison with State-of-the-Art}
We compared STITCHER to two state-of-the-art algorithms: GCOPTER \cite{Wang22:Geometrically} and FASTER \cite{Tordesillas22:FASTER}. 
GCOPTER performs an online optimization by incorporating state constraints into the cost functional and running a quasi-newton method, while FASTER solves a mixed integer quadratic program online. 
Both algorithms rely on a sparse geometric path for safe corridor formation, but do not enforce final trajectories to pass through waypoints. 
We evaluate each planner by time (planning time versus execution time) and failure (constraint violation or incomplete/no path found).
For the Perlin Noise environment, the path lengths were 12.5 m, 30 m, and 55 m, and the path lengths for Willow Garage were 20 m, 25 m, and 30 m with $N =$ 4, 6, 8 waypoints for both environments.

\subsubsection{Time Analysis}
\Cref{tab:lqmt_SOA_comp} compares the planning times and the trajectory execution times of each planner.
STITCHER's planning times are faster than those measured for GCOPTER and FASTER in each test, with an average of 6x and 200x speed up, respectively.
In some cases GCOPTER and FASTER achieved lower execution times, but this was found to be a result of waypoints being treated as soft constraints, i.e., the trajectory is only required to pass nearby a waypoint rather than through it, as well as the chosen resolution of state samples in STITCHER. 

\begin{table}[t]
\caption{State-of-the-Art Comparison Time Analysis}
\vskip -0.1in  
\label{tab:lqmt_SOA_comp}
\begin{center}
\begin{tabular}{c|c|c|c|c|c|c|c}
\hline
\multirow{2}{*}{Map}  & \multirow{2}{*}{$N$} & \multicolumn{3}{|c|}{Planning time (ms)} &\multicolumn{3}{|c}{Execution time (s)}\\
& & \cite{Tordesillas22:FASTER} & \cite{Wang22:Geometrically} &  Ours &  \cite{Tordesillas22:FASTER} & \cite{Wang22:Geometrically} &  Ours  \\
\hline
\multirow{3}{*}{\parbox{0.9cm}{\centering Perlin Noise}} & 4 &  112 & 26.7 & \textbf{3.21} & 3.14 & 3.51 & 3.13\\
& 6  & 233 & 74.7 & \textbf{9.73} & 4.42 & 4.42 & 5.19\\
& 8 & 627 & 121 & \textbf{15.9} & 7.63 & 7.35 & 9.12 \\
\hline
\multirow{3}{*}{\parbox{0.9cm}{\centering Willow Garage}} & 4 & 499 & 30.7 & \textbf{8.24} & 4.63 & 4.38 & 4.43 \\
& 6 & 4030 & 50.5 & \textbf{18.2} & 7.46 & 5.84 & 6.20 \\
& 8 & 23300 & 120 & \textbf{25.9} & 14.7 & FAILED & 7.96\\
\hline
\end{tabular}
\end{center}
\vskip -0.15in  
\end{table}

\begin{table}[t!]
\caption{State-of-the-Art Comparison Failure Analysis}
\vskip -0.1in  
\label{tab:SOA_failure}
\begin{center}
\begin{tabular}{c|c|c|c|c|c|c|c|c|c}
\hline
\multirow{2}{*}{Map} & \multicolumn{3}{|c|}{\parbox{1.8cm}{\vspace{0.05cm}\centering No Path Found (\%)\vspace{0.05cm}}}  & \multicolumn{3}{|c|}{\parbox{1.75cm}{\vspace{0.05cm}\centering Const.\,Violation (\%)\vspace{0.05cm}}} & \multicolumn{3}{|c}{\parbox{1.25cm}{\vspace{0.05cm}\centering Collisions (\%)\vspace{0.05cm}}} \\
& \cite{Tordesillas22:FASTER} & \cite{Wang22:Geometrically} &  Ours &  \cite{Tordesillas22:FASTER} & \cite{Wang22:Geometrically} &  Ours &
\cite{Tordesillas22:FASTER} & \cite{Wang22:Geometrically} &  Ours \\
\hline
\parbox{0.82cm}{\vspace{0.08cm} \centering Perlin Noise\vspace{0.08cm}} & 0 & 0 & \textbf{0} & 6 & 0 & \textbf{0} & 4  & 4 & \textbf{0} \\
\hline
\parbox{0.82cm}{\vspace{0.08cm}\centering Willow Garage\vspace{0.08cm}} & 18 & 2 & \textbf{2} & 8 & 0 & \textbf{0} & 0 & 24 & \textbf{0} \\
\hline
\end{tabular}
\end{center}
\end{table}

\begin{figure}[t]
 \centering
 \includegraphics[width=\columnwidth]{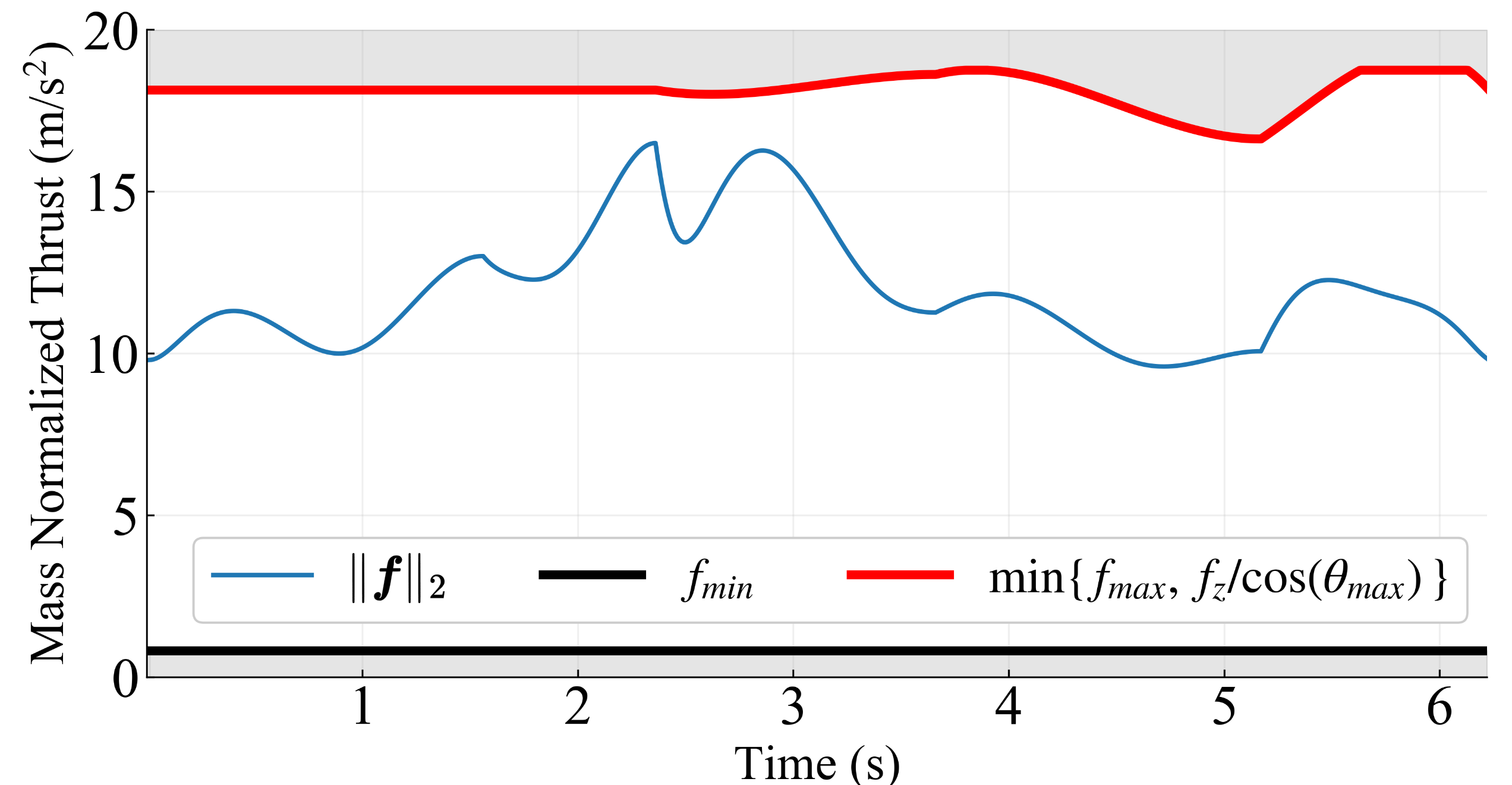}
 \caption{Mass normalized thrust plot depicting constraint satisfaction.}
\label{fig:mass_norm_thrust}
\vskip -0.1in    
\end{figure}

\subsubsection{Failure Analysis}
A Monte Carlo simulation composed of 50 test cases was performed to evaluate the different modes of failure experienced by each planner. 
\Cref{tab:SOA_failure} compares the rate at which each planner does not find a path, generates a trajectory violating constraints or generates a trajectory in collision. 
The ``No Path Found" metric includes a numerical solver not returning a solution, or if the solution does not reach the goal. 
Across all test cases, STITCHER's motion primitive graph disconnects only once, achieving the lowest rate of failure among the tested planners. 
In the Willow Garage environment, where narrow corridors make collisions more likely, the number of failed solutions by FASTER and collisions by GCOPTER significantly increases. 
In contrast, STITCHER never violates constraints (state, control, or obstacles) because all constraints are strictly enforced. 
As an example, \cref{fig:mass_norm_thrust} is a representative mass-normalized thrust profile generated by STITCHER, which remains within the valid limits.

\section{CONCLUSION}
In this work, we presented STITCHER, a motion primitive search planning algorithm that utilizes a novel three-stage planning architecture to design trajectories in real-time over long distances.
We proved the search graph is finite, and the proposed search heuristic is admissible, so STITCHER is guaranteed to i) have \textit{a priori} bounded time and memory complexity and ii) generate optimal trajectories with respect to the sampled set of states. 
Real-time search speeds were achieved through our novel heuristic crafting technique, greedy graph pre-processing method, and non-conservative constraint and collision checking procedure.
Our simulation study showed the effectiveness of the proposed heuristic, the average computation of the components that make up STITCHER, and the satisfaction of complex actuator constraints.
Critically, STITCHER was shown to consistently generate trajectories faster than two state-of-the-art optimization-based planners, with improvements of up to two orders of magnitude for computation time. 
Future work includes developing a receding horizon planning framework, using learning for motion primitive generation and heuristic construction, and hardware/field experiments.

\addtolength{\textheight}{0cm}   



\vskip 0.1in  
\noindent \textbf{Acknowledgments} The authors would like to thank Grace Kwak and Ryu Adams for implementation support.


\bibliographystyle{IEEEtran}
\bibliography{ref}

\end{document}